\documentclass[letterpaper, 10pt, conference]{ieeeconf} 
\usepackage{amsmath}
\usepackage{amssymb}
\usepackage[utf8]{inputenc}
\usepackage{graphicx}
\usepackage{amsthm}
\usepackage{url}
\usepackage[ruled,vlined]{algorithm2e}
\usepackage{subcaption}
\usepackage{xcolor}
\usepackage{microtype}
\usepackage{booktabs}
\usepackage{multirow}
\usepackage{siunitx}
\usepackage{gensymb}
\usepackage{grffile}
\usepackage{tikz}

\IEEEoverridecommandlockouts
\title{\LARGE \bf Geometry-based Graph Pruning for Lifelong SLAM\vspace{-5mm}}
\author{Gerhard Kurz, Matthias Holoch, Peter Biber
\thanks{The authors are with Robert Bosch GmbH, Corporate Research, Germany. E-mail: \{gerhard2.kurz,  matthias.holoch, peter.biber\}@de.bosch.com \newline We thank Kai O. Arras and Artur Koch for their helpful input.}
\vspace{-2cm}
}

\newtheorem{definition}{Definition}
\newtheorem{theorem}{Theorem}

\DeclareMathOperator*{\argmax}{arg\,max}

\SetCommentSty{mycommfont}

\newcommand\copyrighttext{%
	\footnotesize \textcopyright 2021 IEEE. Personal use of this material is permitted.
	Permission from IEEE must be obtained for all other uses, in any current or future
	media, including reprinting/republishing this material for advertising or promotional
	purposes, creating new collective works, for resale or redistribution to servers or
	lists, or reuse of any copyrighted component of this work in other works.
	}
\newcommand\copyrightnotice{%
	\begin{tikzpicture}[remember picture,overlay]
	\node[anchor=south,yshift=10pt] at (current page.south) {\fbox{\parbox{\dimexpr\textwidth-\fboxsep-\fboxrule\relax}{\copyrighttext}}};
	\end{tikzpicture}%
}

\begin{document}
\maketitle
\copyrightnotice
\thispagestyle{empty}
\pagestyle{empty}

\newcommand{\edge}[2]{(#1,#2)}

\begin{abstract}
	Lifelong SLAM considers long-term operation of a robot where already mapped locations are revisited many times in changing environments.
	As a result, traditional graph-based SLAM approaches eventually become extremely slow due to the continuous growth of the graph and the loss of sparsity.
	Both problems can be addressed by a graph pruning algorithm.
	It carefully removes vertices and edges to keep the graph size reasonable while preserving the information needed to provide good SLAM results.
	We propose a novel method that considers geometric criteria for choosing the vertices to be pruned.
	It is efficient, easy to implement, and leads to a graph with evenly spread vertices that remain part of the robot trajectory.
	Furthermore, we present a novel approach of marginalization that is more robust to wrong loop closures than existing methods.
	The proposed algorithm is evaluated on two publicly available real-world long-term datasets and compared to the unpruned case as well as ground truth.
	We show that even on a long dataset (25h), our approach manages to keep the graph sparse and the speed high while still providing good accuracy (40 times speed up, 6cm map error compared to unpruned case).
\end{abstract}

\section{Introduction}
Graph-based approaches are widely considered the most popular solution to the simultaneous localization and mapping (SLAM) problem~\cite{Grisetti2010a} due to their reliability, flexibility, and favorable computational cost~\cite[Sec. II]{Cadena2016}.
In \emph{lifelong} SLAM~\cite{Krajnik2016}, we assume that an area is not just mapped once but revisited continuously over a long period of time.
This leads to a number of additional challenges compared to a standard SLAM problem where an area is only covered once or twice during a short period of time~\cite[Sec. III \& IV]{Cadena2016}.
This paper focuses on one of these challenges: scalability. The graph grows indefinitely over time, which makes the SLAM algorithm prohibitively slow and implies virtually unlimited memory requirements.

To illustrate the problem, imagine an environment of a limited size.
A robot keeps moving around in this environment for a long time while running a naive graph-based SLAM algorithm that continuously adds new measurements to its map.
The number of vertices in the graph as well as the memory requirements will typically increase roughly linearly w.r.t. the time spent driving.
As parts of the environment are traversed again and again, more and more loop closure edges will get added between nearby vertices.
In the worst case, this leads to a quadratic number of edges w.r.t. time and to a graph whose structure is not sparse anymore.
This impacts the runtime severely, because commonly used optimizers such as g2o \cite{Kuemmerle2011} or gtsam \cite{Dellaert2012} are only efficient if the graph is sparse and become extremely slow for dense graphs.

The goal of graph pruning is to carefully remove redundant vertices and edges, while minimizing the loss in localization and mapping accuracy.
The algorithm ensures that the size of the graph only scales with the size of the environment and becomes independent of the time spent there or the length of the robot's trajectory.
An example of the effects of pruning is shown in Fig.~\ref{fig:flatexample}.

\begin{figure}
	\centering
	\includegraphics[height=50mm, trim=0 0 0 7mm, clip]{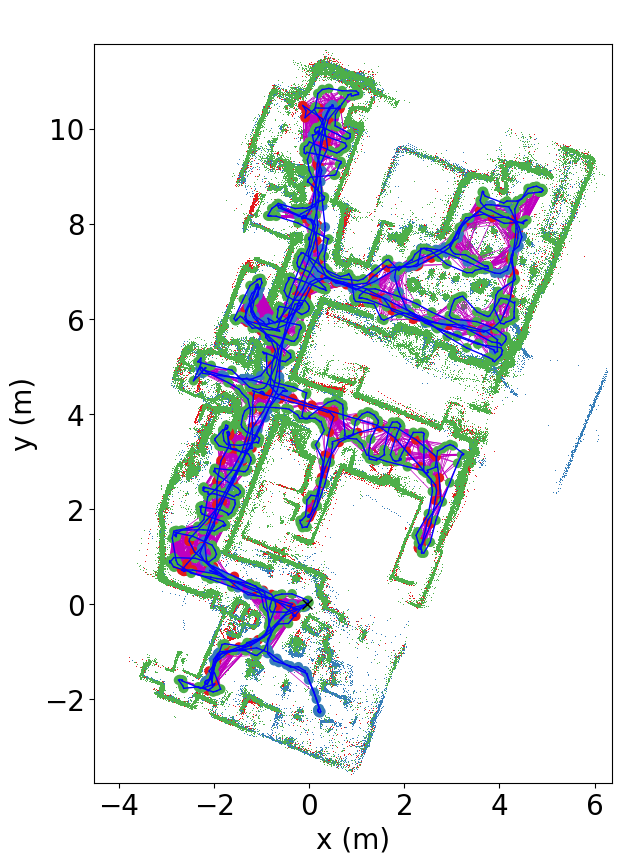}
	\quad
	\includegraphics[height=50mm, trim=0 0 0 7mm, clip]{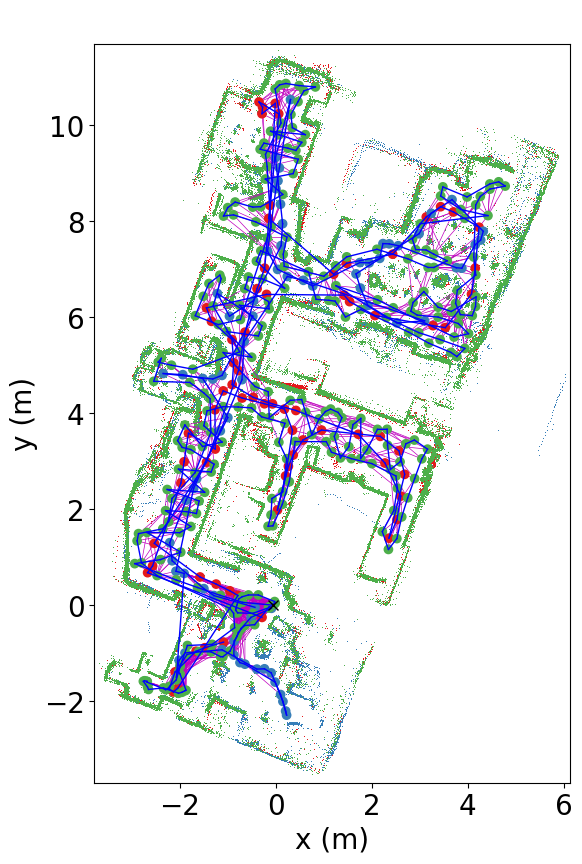}
	\caption{SLAM map of an apartment without (left) and with pruning (right) using the proposed approach.}
	\label{fig:flatexample}
\end{figure}

In this paper, we present a novel approach that performs vertex pruning based on geometric criteria. 
Specifically, we make the following contributions:
\begin{enumerate}
	\item An efficient approach for vertex pruning based on a \emph{scale-invariant density} that does not depend on the data associated to a given vertex and leads to an even distribution of vertices in the pruned graph.
	\item A marginalization algorithm that is robust to incorrect loop closures.
	\item A thorough evaluation on two real-world datasets using different pruning parameters.
	\item An evaluation scheme that does not require ground truth and is independent of the particular SLAM algorithm by comparing the graph with and without pruning.
\end{enumerate}

\section{Related work}

Several other authors have considered the problem of graph pruning before.
Kretschmar et al.~\cite{Kretzschmar2010} proposed a pruning scheme where the information gain of each lidar scan is computed using a probabilistic occupancy grid.
Then, the vertices whose lidar scans provide the least additional information are removed.
Marginalization of vertices is performed by replacing all edges of the pruned vertex with edges connected to one of its neighbors.
An improved version of the paper uses a Chow--Liu tree~\cite{Chow1968} to approximate the clique obtained by marginalization~\cite{Kretzschmar2011}. 

Eade et al. proposed a similar graph pruning approach~\cite[Sec.~VIII]{Eade2010} in the context of visual SLAM.
Nodes in the graph are divided into view and pose nodes, the latter of which can be pruned after a while because they only encode the robot's trajectory and constraints in the graph but are not relevant for the map itself.
Marginalization is performed by inserting binary edges between all vertices adjacent to a pruned vertex, which leads to the creation of numerous new edges.
Hence, the authors proposed an edge pruning algorithm that uses a simple heuristic to remove superfluous edges.

Lazaro et al. suggested a somewhat different approach where multiple scans are aggregated into a local map, which corresponds to a single vertex in the graph~\cite{Lazaro2018}.
Later, several local maps can be merged into a single vertex by building on the idea of condensed measurements previously proposed by Grisetti et al.~\cite{Grisetti2012}.
Their approach effectively combines pruning and change detection into a single operation.
A disadvantage is that newly created vertices are not located on the robot's trajectory anymore.

Carlevaris-Bianco et al. have further investigated the issue of marginalizing vertices from the graph~\cite{CarlevarisBianco2013,CarlevarisBianco2014}.
They propose an exact marginalization using n-ary factors and a sparse approximation, which is based on a Chow--Liu tree.
This approach is different from~\cite{Kretzschmar2011}, where the tree-based approximation is computed from binary factors.
The benefit of considering n-ary factors is that correlations between binary factors are not ignored and that the approach can also be applied to factors that do not constrain the full state (e.g., bearing only measurements, range only measurements).
Later works by Mazuran et al.~\cite{Mazuran2015} and Ta et al.~\cite{Ta2018} generalize and improve this approach further.

A pruning approach for feature-based SLAM was proposed by~\cite{Wang2015}.
They use the Kullback--Leibler divergence between the graph before and after pruning to decide which vertex to prune.
For marginalization, they consider different ways to approximate the clique using fewer edges.
Unfortunately, their work is not directly applicable to non-feature based SLAM, e.g., lidar SLAM based on scan matching.

\section{Vertex Pruning}

\begin{figure*}
	\centering
	\newcommand{\figwidthgrid}{39mm}
	\newcommand{\figheightgrid}{38mm}
	\newcommand{\figwidth}{45mm}
	\begin{subfigure}{.3\linewidth}
		\includegraphics[width=\figwidthgrid]{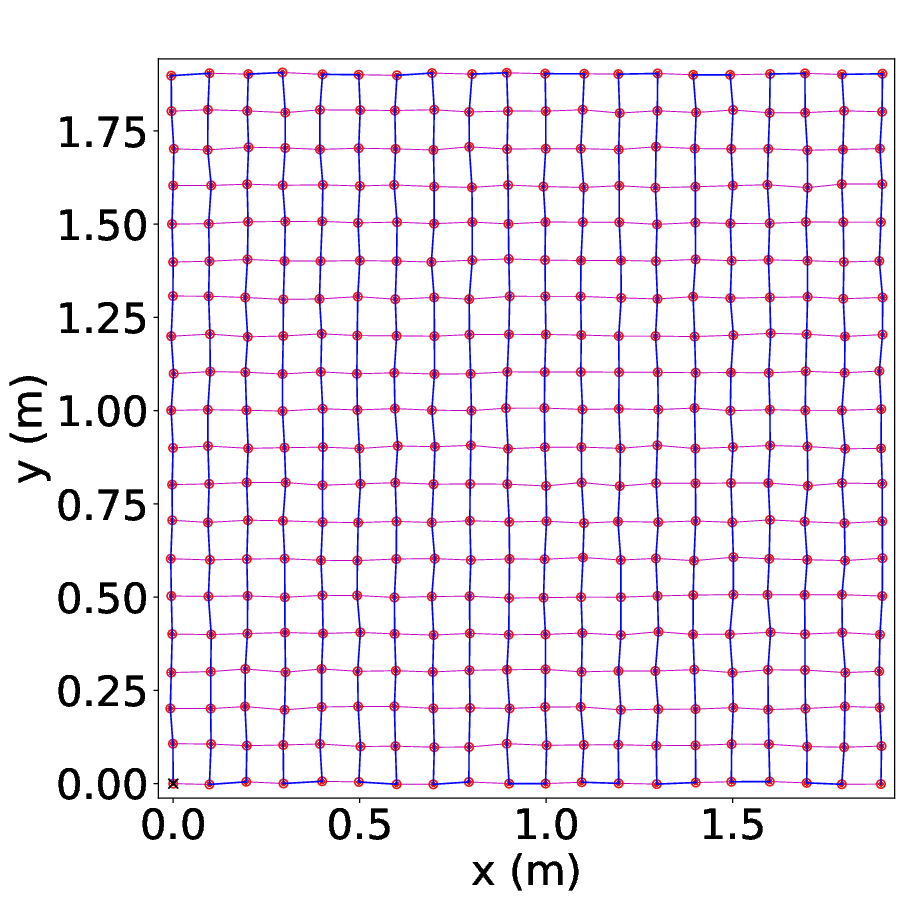} \\
		\includegraphics[trim=0 0 41mm 0, clip, height=\figheightgrid]{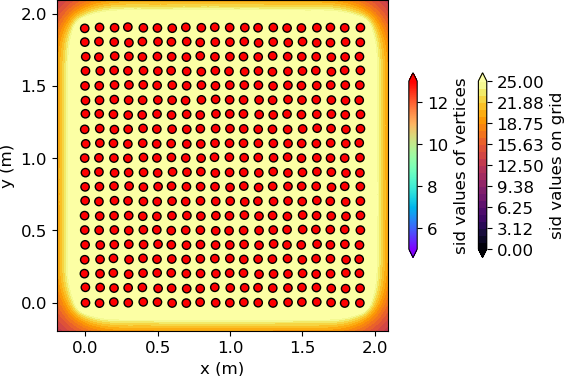}
		\captionsetup{margin={-5mm, 0cm}} 
		\caption{grid, original}
	\end{subfigure}
	\begin{subfigure}{.3\linewidth}
		\includegraphics[width=\figwidthgrid]{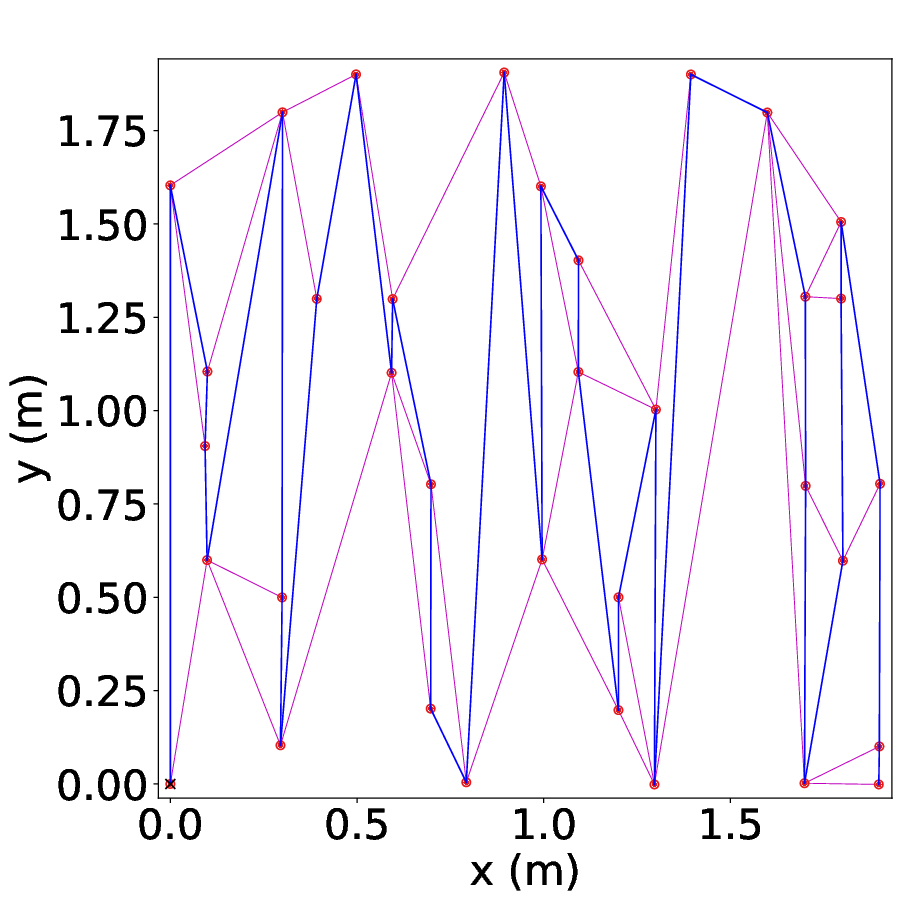}\\
		\includegraphics[trim=0 0 41mm 0, clip, height=\figheightgrid]{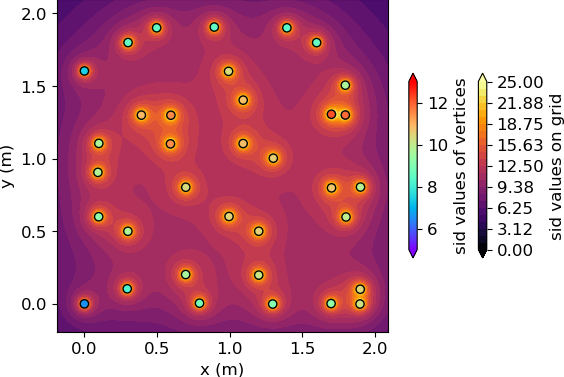}
		\captionsetup{margin={-8mm, 0cm}} 
		\caption{grid, density-based}
	\end{subfigure}
	\begin{subfigure}{.36\linewidth}
		\includegraphics[width=\figwidthgrid]{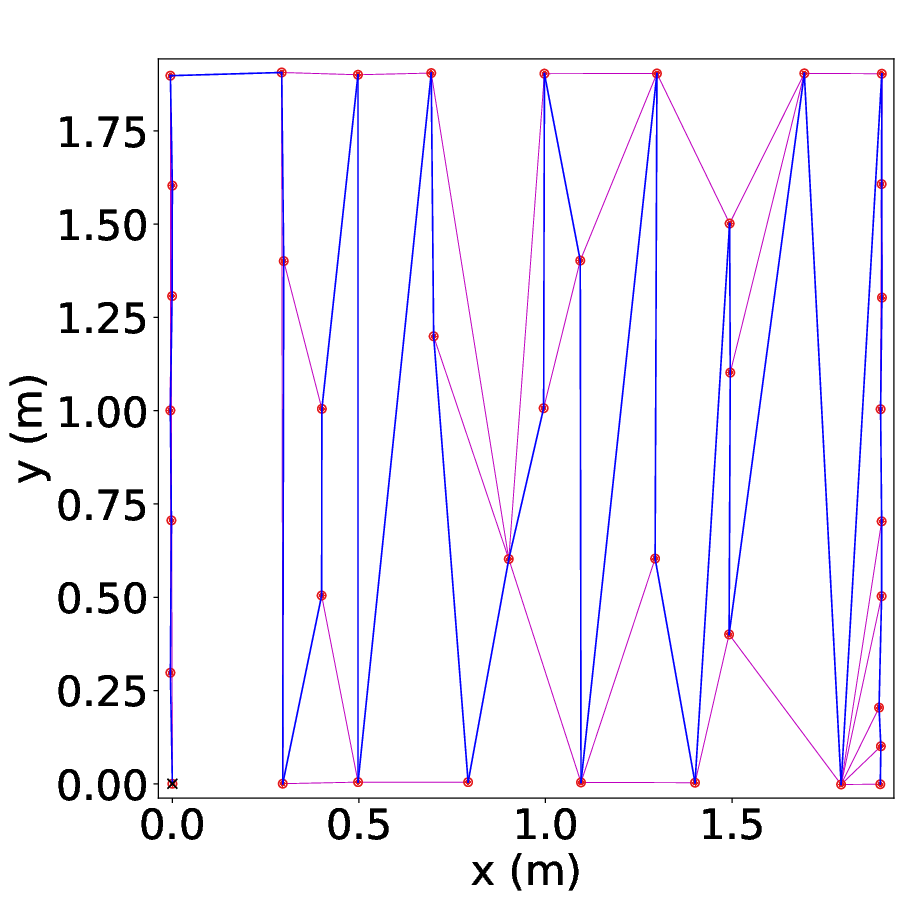} \\
		\includegraphics[height=\figheightgrid]{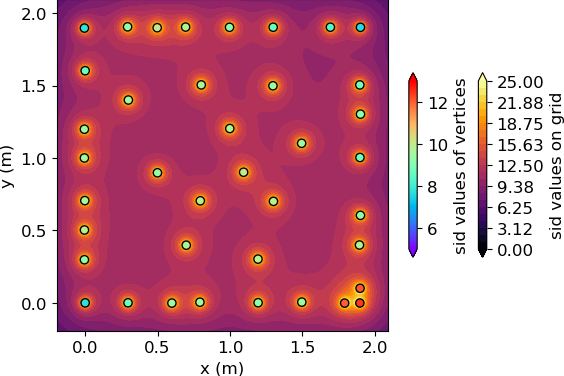}
		\captionsetup{margin={-15mm, 0cm}} 
		\caption{grid, sid-based}
	\end{subfigure} 
	\\
	\newcommand{\figheight}{28.5mm}
	\begin{subfigure}{.3\linewidth}
		\includegraphics[width=\figwidth]{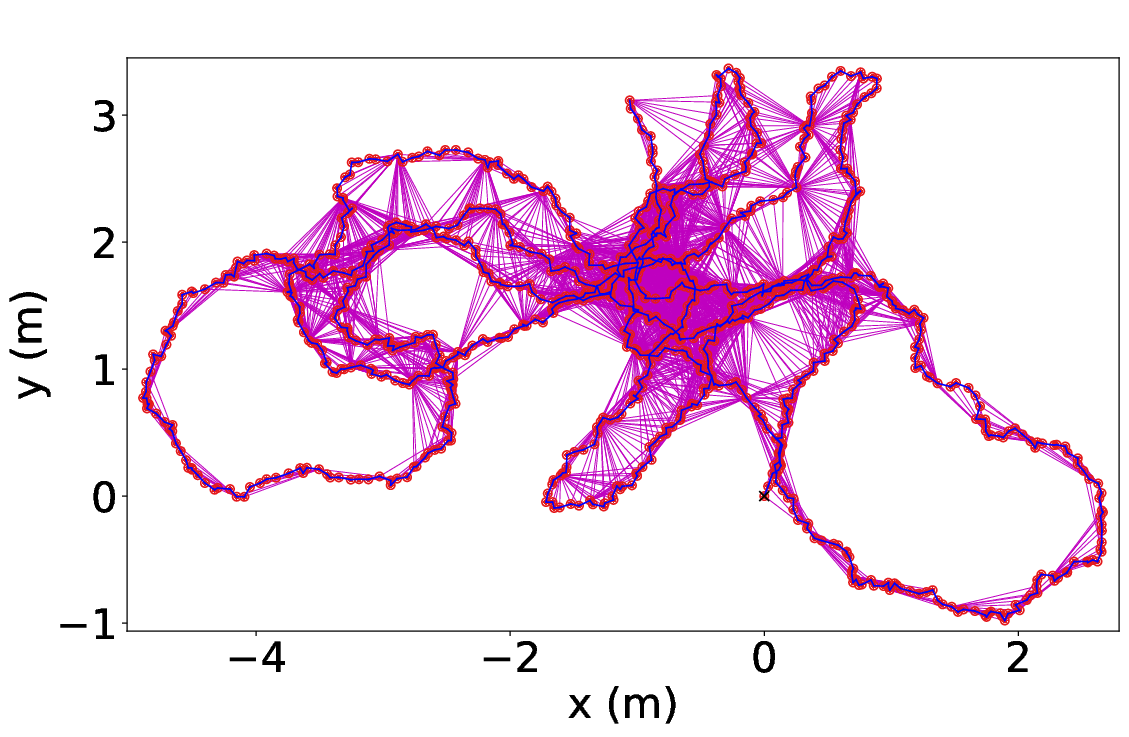} \\
		\includegraphics[trim=0 0 41mm 0, clip, height=\figheight]{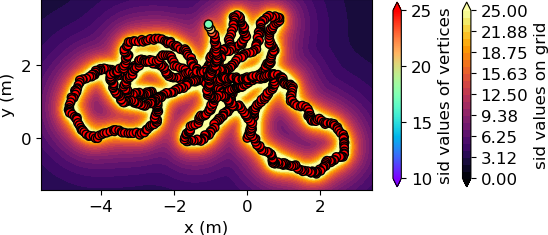}
		\captionsetup{margin={-5mm, 0cm}} 
		\caption{random, original}
	\end{subfigure}
	\begin{subfigure}{.3\linewidth}
		\includegraphics[width=\figwidth]{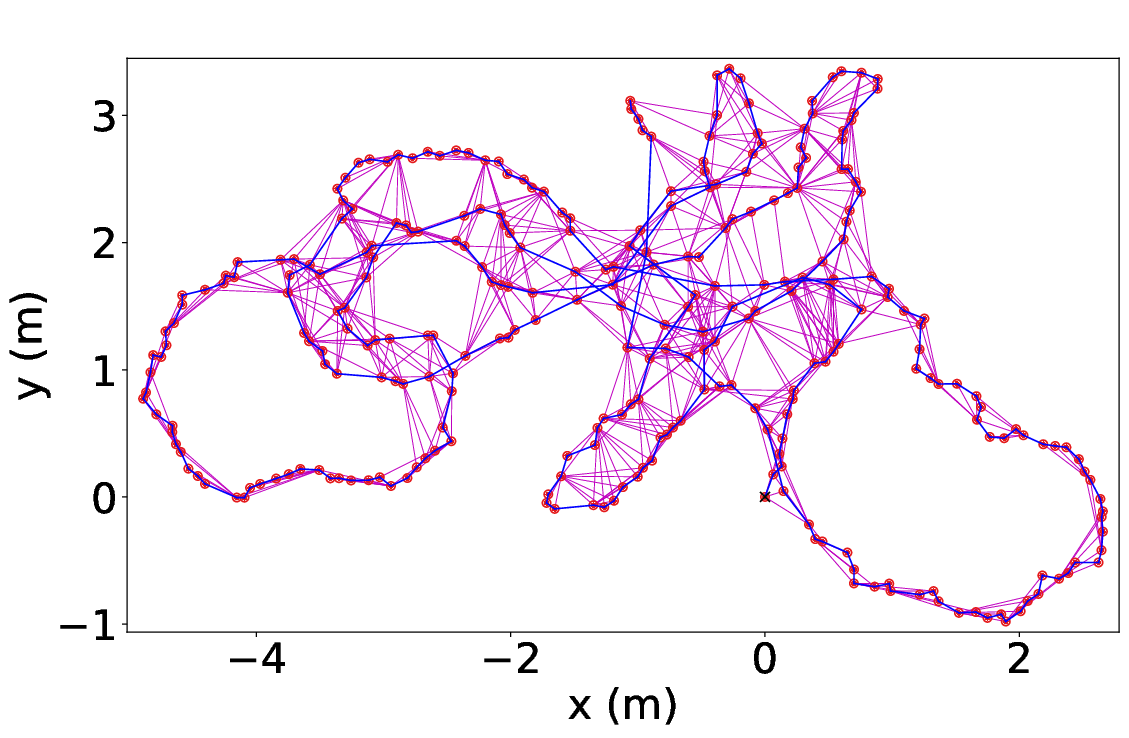}\\
		\includegraphics[trim=0 0 41mm 0, clip, height=\figheight]{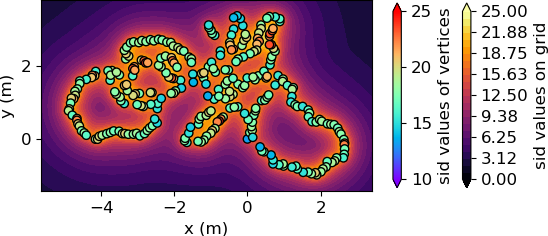}
		\captionsetup{margin={-8mm, 0cm}} 
		\caption{random, density-based}
	\end{subfigure}
	\begin{subfigure}{.36\linewidth}
		\includegraphics[width=\figwidth]{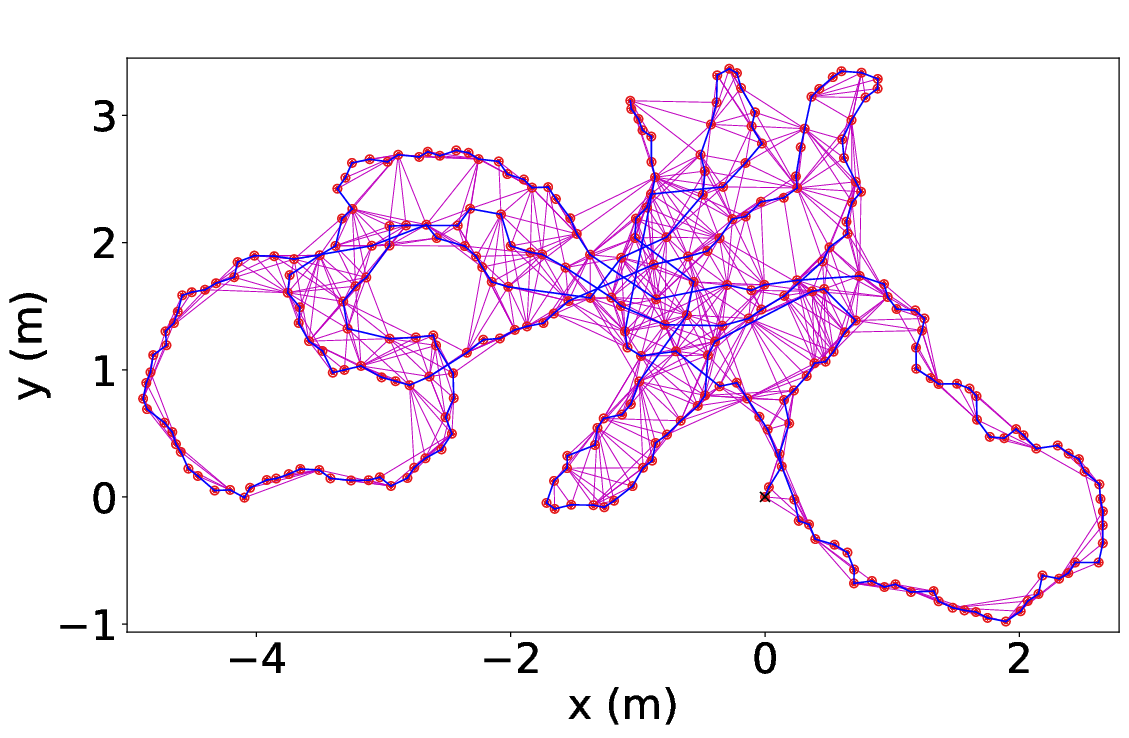}
		\includegraphics[height=\figheight]{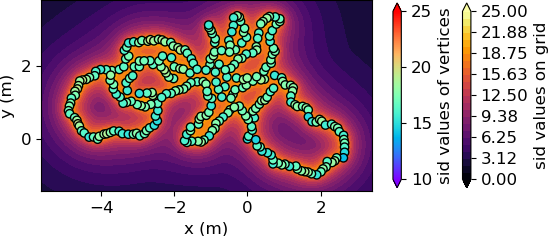}
		\captionsetup{margin={-15mm, 0cm}} 
		\caption{random, sid-based}
	\end{subfigure} 
	\caption{Vertex pruning approaches on synthetic examples (regular grid and a random trajectory).
    Odometry edges are shown in {\color{blue}blue} and loop closure edges are shown in {\color[HTML]{bf00bf} magenta}.
    The original graph is pruned using a simple density-based approach with $r=\SI{0.45}{\meter}$ (see Definition~\ref{def:radius}) or using the scale-invariant density (see Definition~\ref{def:sid}).
    Observe that the proposed sid-based approach yields more evenly spread vertices.}
	\label{fig:vertexpruning}
\end{figure*}

In order to keep the number of vertices inside the graph limited, it is necessary to carefully remove vertices over time.
In general, there are two strategies, selecting one or more vertices and 1) marginalizing them~\cite{Kretzschmar2010,Kretzschmar2011,Eade2010}, or 2) fusing them into one~\cite{Lazaro2018}.
In lidar SLAM, the second option has the disadvantage that the lidar scans have to be fused and that pose of the new vertex does not match the pose where the lidar scans were originally observed.
This complicates or invalidates potential successive processing steps such as raytracing of free space using scan poses as, e.g., proposed for change detection in~\cite{Underwood2013}.
In addition, all further information that may be associated with the involved vertices has to be fused as well.
Finally, the fused vertex is not part of the original trajectory and may be located on impassable terrain.
For these reasons, we only consider the first option in the following.

Hence, the key question is: How do we choose which vertices to remove? In general, one could consider the simultaneous removal of multiple vertices and analyze the combined effect of this operation.
However, this is difficult to analyze and for the sake of simplicity, we only consider removing a single vertex at a time.
When choosing which vertex to remove, we would like to pick a vertex that is of low importance to the overall map quality and whose removal degrades the expected future performance of the SLAM algorithm as little as possible.
In order to quantify the importance of a vertex, geometric or information-theoretic measures can be used.

Information-theoretic measures try to quantify the information contained in each vertex or, as proposed by Kretschmar et al.~\cite[Sec.~4.1]{Kretzschmar2010}, \cite[Sec.~IV]{Kretzschmar2011}, its associated laser scan.
The idea is then to remove vertices in such a way that the loss of information is minimized.
The main problem with this approach is that the actual information within a scan is hard to quantify.
Kretschmar et al. compute the information based on the occupation probabilities of cells in a probabilistic grid map.
However, this approach tends to consider scans with high noise, false measurements, or poor alignment to be high in information because they affect the probabilities inside the grid map more strongly than well aligned scans with little noise that agree well with other scans of the environment.
Also, this information measure is fairly expensive to compute because all points in scans within range of the current vertex contributed to the result.
This can be problematic especially for lidars with long range and high resolutions.

For these reasons, we propose to use geometric methods for vertex pruning.
The basic idea is to keep the density of vertices below a certain threshold across the entire map, i.e., we seek to remove vertices in places where a lot of vertices cluster within a small area.
For this purpose, we propose a novel density measure, which we call \emph{scale-invariant density}.

\subsection{Scale-invariant density}

Consider vertices with unique positions $v_1, \dots, v_n \in \mathbb{R}^2$, where $v_i \neq v_j$ for $i\neq j$. 

\begin{definition}
	\label{def:radius}
	For a radius $r > 0$, define the $r$-density of a vertex $v_i$ as
	\begin{align*}
	d_r(v_i) = \frac{ \# \{ \lVert v_j - v_i \rVert < r: 1\leq j \leq n , i \neq j \}  }{\pi r^2} \ ,
	\end{align*}
	where $\lVert\cdot\rVert$ is the Euclidean norm. This corresponds to the number of vertices (excluding $v_i$) inside the circle of radius $r$ around $v_i$ divided by the circle's area.
\end{definition}
However, this measure is strongly dependent on the choice of the radius $r$.
In particular, after pruning vertices inside the radius $r$, there may be a lot vertices whose distance from $v_i$ is just slightly larger than $r$, which leads to underestimating the actual density of vertices around $v_i$.

In a pruning algorithm, this can leave behind clusters of vertices that are just slightly more than $r$ apart.
This issue is illustrated in Fig.~\ref{fig:vertexpruning}, where we consider the effect of vertex pruning on two synthetic examples.

To address the deficiencies of the radius-based density $d_r(v_i)$, we propose the \emph{scale-invariant density}, which is independent of any fixed scale $r$. It is obtained by integrating over all $r \in (0, \infty)$. 
\begin{definition}
	\label{def:sid}
	The scale-invariant density of $v_i$ is given by
	$$
	d(v_i) = \int_0^\infty d_r(v_i) dr \ .
	$$
\end{definition}
As a result, this measure is not susceptible to the clustering problem mentioned above.
In order to efficiently compute the density measure, we use the following theorem.

\begin{theorem}
	\label{thm:sid}
	It holds that
	$$
	d(v_i) = \frac{1}{\pi} \sum_{k \in \{1,\dots,n \} \backslash \{i\} }  \lVert v_k - v_i \rVert^{-1}	
	$$
\end{theorem}
\begin{proof}
  Wlog., we consider the parameter $v_n$ and assume that all other $v_i$ for $i = 1, \dots, n-1$ are sorted by distance $r_i =  \lVert v_i - v_n \rVert $ according to $r_1 \leq \dots \leq r_{n-1}$.
  Then, we can simplify
\allowdisplaybreaks
\begin{align*}
	d(v_n) =& \int_0^\infty d_r(v_n) dr \\
	=& \left(\sum_{k=1}^{n-2} \int_{r_k}^{r_{k+1}}  d_r(v_n) dr \right) + \int_{r_{n-1}}^\infty d_r(v_n) dr \\
	=&	\left(\sum_{k=1}^{n-2} \int_{r_k}^{r_{k+1}} \frac{k}{\pi r^2} dr\right) + \int_{r_{n-1}}^\infty \frac{n-1}{\pi r^2} dr \\
	=& \left(\sum_{k=1}^{n-2} \frac{k}{\pi} \left(r_k^{-1} - r_{k+1}^{-1} \right) \right) + \frac{n-1}{\pi} r_{n-1}^{-1} \\
	=& \frac{1}{\pi} \sum_{k=1}^{n-1} r_k^{-1} \qedhere
\end{align*}
\end{proof}
\newcommand{\sidtruncation}{\hat{N}}
For practical use, this sum can be truncated to the closest $\sidtruncation \in \mathbb{N}$ neighbors, which speeds up computation and leads to a more localized version of the density.
Based on Theorem~\ref{thm:sid}, it is easy to show that scaling the positions of the vertices corresponds to inversely scaling the density.

\subsection{Vertex Pruning Algorithm}
\newcommand{\thresholdvertices}{\hat{n}}
\newcommand{\thresholdsid}{\hat{s}}
\newcommand{\thresholdodochain}{\hat{m}}

Based on the density function introduced above, we propose a vertex pruning algorithm.
This algorithm tries to find the vertex with highest density, marginalizes it and repeats until the graph size is below a predefined threshold $\thresholdvertices \in \mathbb{N}$ or no more vertices have a density below the density threshold $\thresholdsid \in (0,\infty)$.
Furthermore, it can be desirable to exclude certain vertices from pruning.
For example, we always keep the most recent $\thresholdodochain \in \mathbb{N}$ vertices to prevent pruning vertices prematurely, which could lead to poor robustness.
We define the set of prunable vertices as $\tilde{V} \subseteq V$ and only consider theses vertices as candidates for pruning.
Note that the density $d(v_i)$ is still calculated using all vertices including non-prunable ones.
Pseudocode of the proposed algorithm is given in Algorithm~\ref{algo:vertexpruning}.

\begin{algorithm}[t]
	\caption{Vertex Pruning}
	\label{algo:vertexpruning}
	\KwIn{graph $(V,E)$, thresholds $\thresholdvertices, \thresholdsid$}
	\KwOut{updated graph $(V,E)$}
	\Repeat{{$\#\tilde{V} \leq \thresholdvertices$}}
	{
		$\tilde{V} \gets \{ v \in V : v$ can be pruned$\}$\;
		\tcc{Find vertex with largest sid}
		$v_\text{max} \gets \argmax_{v \in \tilde{V}} d(v)$\;
		\tcc{Remove vertex if density above $\thresholdsid$}
		\eIf{$d(v_\text{max}) > \thresholdsid$}
		{
			$(V,E) \gets $ marginalize$\big((V,E), v_\text{max}\big)$ \;	
		}
		{
			break\;
		}
	}
	\Return $(V,E)$\;
\end{algorithm}

It should be noted that the proposed algorithm does not consider the direction the robot was facing at each vertex.
This is fine for robots whose field of view is 360\degree or close to that.
However, for robots with a small field of view it may be necessary to first partition the vertices by viewing angle and prune each partition separately.

\subsection{Marginalization}
Once we have chosen a vertex to prune from the graph, we need to consider how to perform the removal of a given vertex.
The graph encodes the joint probability distribution of the random variables represented by all vertices.
Thus, removing a vertex corresponds to marginalizing a set of random variables.

Marginalization is typically performed by creating an n-ary constraint between all neighbors of the pruned vertex~\cite[Sec.~II]{CarlevarisBianco2014}.
Some authors approximate this n-ary constraint by connecting all neighbors of the pruned vertex pair-wise, i.e., vertices adjacent to the pruned vertex will form a clique.
This leads to a huge number of edges if the pruned vertex has a lot of neighbors or if several vertices are pruned successively~\cite[Sec.~V-A]{Kretzschmar2011}.
As a result, this method for marginalization is infeasible in practice, at least unless an aggressive edge pruning scheme is employed to remove most of these edges later (see Sec.~\ref{sec:edgepruning}).

To obtain a more sparse approximation of the graph in the neighborhood of the pruned vertex, several authors have proposed using a Chow--Liu tree~\cite{Chow1968} either directly from an n-ary constraint~\cite{CarlevarisBianco2014} or by selecting a subset of the fully connected graph of binary constraints among all neighbors~\cite[Sec.~V-C]{Kretzschmar2011}.
Either way, the Chow--Liu tree approximates the neighborhood of the marginalized vertex using a spanning tree weighted by the mutual information between the vertices involved.
As the local tree structure might be too sparse in practice, other subgraphs have been considered, e.g., the Cliquey approach~\cite{Mazuran2015}.

\subsubsection{Impact of Wrong Loop Closures on Marginalization}
An important issue that has not been sufficiently considered in related work on graph pruning is the influence of incorrect loop closures during marginalization.
Typical SLAM algorithms can be made resilient to a small number of incorrect loop closures by using robust kernels in the error function or by specifically determining which loop closures are outliers using an expectation maximization algorithm~\cite{Lee2013}.
However, marginalization can amplify wrong loop closures. After multiple marginalization steps, even a single wrong loop closure might be amplified enough to lead to divergence.
This is particularly problematic because wrong loop closures often do not have higher uncertainty according to their information matrix than correct loop closures because the scan matching algorithm only provides a local uncertainty estimate.

Consider the following example.
We would like to prune a vertex $v_0$ with $n$ edges to $v_1,\dots, v_n$.
Let us assume that a single edge $\edge{v_0}{v_i}, i \in \{1, \dots, n\}$ is incorrect, i.e., the ratio of incorrect edges is $1/n$.
After marginalization of $v_0$, the combined n-ary constraint would be incorrect.
We now approximate this constraint by creating pairwise edges $\edge{v_i}{v_j}$ for $i \neq j \in \{1, \dots, n\}$.
This yields $n\cdot(n-1)/2$ newly created edges, of which $n-1$ are wrong edges.
Therefore, the number of incorrect edges increased from $1$ to $n-1$ and the ratio of incorrect edges increased from $1/n$ to $2/n$.
In addition, if there already was an edge between one of $v_0$'s neighbors $\edge{v_i}{v_j}$ for $i\neq j \in \{1, \dots, n\}$ before marginalization, the existing and new edges would get fused.
This leads to an incorrect edge if either of the original edges was incorrect, i.e., the ratio of incorrect edges would be even larger.

In practice however, we make the observation that odometry edges are typically much more reliable than loop closure edges even if they possess the same uncertainty.
While odometry can drift over time and may be somewhat inaccurate, it is usually not affected by outliers with huge errors like loop closures.
As a result, we propose a marginalization algorithm that is founded on the reliability of the odometry chain and is fairly robust to wrong loop closures.

\subsubsection{Marginalization Algorithm}
In \cite[Sec.~4.2]{Kretzschmar2010}, an approximate marginalization was proposed where all edges of the pruned vertex are collapsed into a single neighboring vertex.
This vertex is chosen such that the sum of lengths of all resulting edges is minimized.
We propose a similar method that also collapses edges into adjacent vertices.
However, we do not collapse all edges into a single vertex but instead move each loop closure forward or backward along the odometry chain depending on the distance between the vertices of the newly created edge.
Pseudocode of this procedure is given in Algorithm~\ref{algo:marginalization}.

\begin{algorithm}[t]
	\caption{Marginalization}
	\label{algo:marginalization}
	\KwIn{graph $(V,E)$, vertex $v \in V$ to be marginalized}
	\KwOut{updated graph $(V,E)$}
	$L \gets \{e \in E : e$ adjacent to $v \land e $ loop closure $ \}$\;
	$e_\text{in} = (v_\text{prev},v) \gets $ odo edge into $v$\;
	$e_\text{out} = (v, v_\text{next}) \gets $ odo edge out of $v$\;
	$V \gets V \backslash \{v\}$\;
	$E \gets E \backslash \{e \in E: e$ adjacent to $v\}$\;
	\tcc{Move loop closure edges along odometry chain}
	\For{$e = (v, v_\text{other}) \in L$} 
	{
		$d_\text{in} = \lVert v_\text{prev} - v_\text{other} \rVert$\;
		$d_\text{out} =  \lVert v_\text{next} - v_\text{other} \rVert$\;
		\eIf{$d_\text{in} < d_\text{out}$}
		{
			$E \gets E \cup \{$concatenateEdges$(e_\text{in}, e) \}$\;
		}
		{
			$E \gets E \cup \{$concatenateEdges$(e, e_\text{out}) \}$\;
		}
	}
	\Return $(V,E)$\;
\end{algorithm}

To concatenate edges, we need three elementary operations, composition of two edges, combination of two edges and inversion of an edge, see \cite[Sec.~VII-B]{Eade2010}.
The equations for these operations depend on the exact error function used in the optimizer (in particular the frame covariances are represented in) and can be computed in closed-form.

In the proposed algorithm, the number of wrong loop closures does not change during marginalization.
Considering the same example as above, our algorithm reduces the number of edges from $n$ to $n-1$, and thus the ratio of incorrect loop closures rises only slightly from $1/n$ to $1/(n-1)$.
When combining edges, we can detect contradictory edges by considering the differences in their relative poses or the Mahalanobis distance.
If an odometry edge and a loop closure contradict, we trust the odometry edge and delete the loop closure.
If two loop closures contradict, we delete them both as losing a correct loop closure is usually preferable to having an incorrect loop closure in the graph.

\subsection{Evaluation}

To investigate the behavior of the proposed algorithm in the presence of wrong loop closures, we created a synthetic example based on a $30 \times 30$ grid similar to Fig.~\ref{fig:vertexpruning} and randomly replaced a given percentage of the loop closures with incorrect loop closures.
Then, we applied pruning using different marginalization methods, optimized the graph again and compared the vertex positions to the true positions.
In Fig.~\ref{fig:false_loop_closure_eval}, we show the results from 50 Monte Carlo runs with different random broken loop closures.
It can be seen that the proposed algorithm even slightly outperforms the original SLAM algorithm without pruning, which is mainly due to the detection and removal of some of the false loop closures when they contradict other edges.
The Chow-Liu approach performs significantly worse and diverges completely in some cases.

\begin{figure}
	\centering
	\includegraphics[width=.8\linewidth]{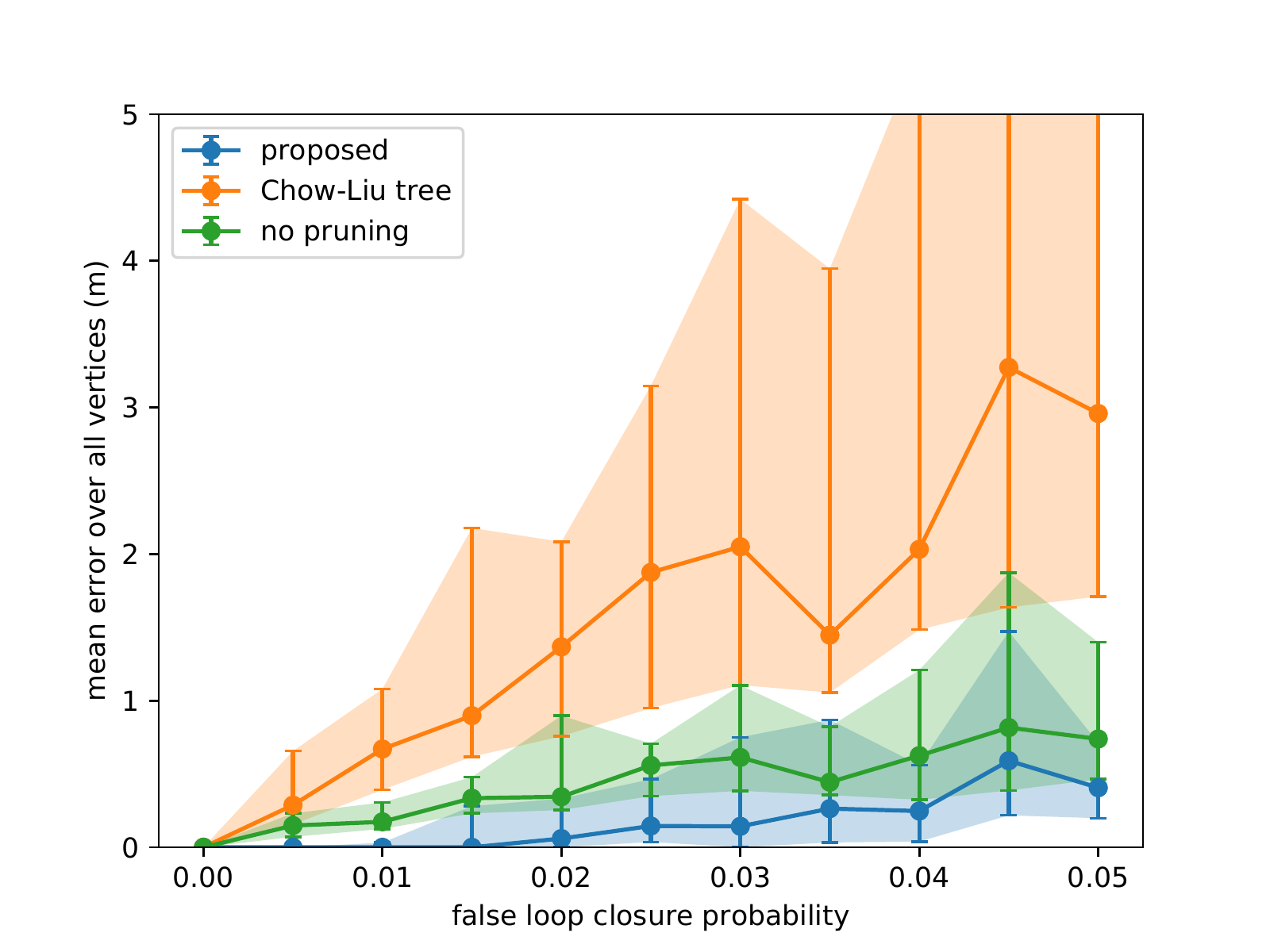}
	\caption{Influence of wrong loop closures on the vertex position error. We show the median, 25\%, and 75\% quantiles over 50 Monte Carlo runs.}
	\label{fig:false_loop_closure_eval}
\end{figure}

\section{Edge Pruning}
\label{sec:edgepruning}
In addition to vertex pruning and the removal of edges during marginalization, we also propose an independent edge pruning step, which removes superfluous edges without touching vertices at all.
This is mostly beneficial in areas where vertices are not particularly dense but a lot of loop closures have been found.
Also, depending on the particular algorithm, a lot of edges may be created during marginalization.
In general, removing an edge is equivalent to removing one constraint from the graph and means losing the information encoded in that edge.

\newcommand{\thresholdedges}{\hat{e}}

To choose which edges to prune, we consider the vertex with the largest number of edges and remove one of its edges via an edge selection criterion.
This step is repeated until all vertices have less than a predefined number of edges $\thresholdedges \in \mathbb{N}$.
The key question is how to choose the edge selection criterion.
Eade et al. employed pruning based on the edge residuals. Edges with the smallest residual are pruned first~\cite[Sec.~VIII-C]{Eade2010}.
However, using the residual (or the $\chi^2$ error) means that edges with large errors (e.g., noisy edges, poor scan matches, wrong loop closures) will be kept and edges with small errors (good matches) will be discarded.

We therefore propose to use the trace or determinant of the edges' covariance/information matrices as an alternative criterion.
As the differences between these measures seem to be small in practice, we decided to use the trace of the information matrix, which is efficient to compute.
We favor keeping edges with more certain estimates over edges with less certain estimates, i.e., edges with little information will be pruned.

\newcommand{\thresholddistancefactor}{\hat{d}}
\newcommand{\shortestpath}{d}
\newcommand{\distancefactor}{\widetilde{d}}

Removing edges comes with the risk of removing important global loop closures or even disconnecting the graph.
To avoid pruning edges that keep the graph together, we consider the factor $\distancefactor$ by which the shortest path gets longer when removing an edge.
Assume we want to prune the edge $\edge{v_i}{v_j}$ with length $\lVert v_i - v_j \rVert$.
If we removed this edge, the length of the shortest path from $v_i$ to $v_j$ would become $\shortestpath_{G \setminus \edge{v_i}{v_j}}(v_i, v_j)$, which can be computed efficiently using the A* algorithm~\cite{Hart1968}.
We therefore compute the ratio $\distancefactor = \frac{\shortestpath_{G \setminus \edge{v_i}{v_j}}(v_i, v_j)}{\lVert v_i - v_j \rVert}$ and only allow pruning if it does not exceed a predefined threshold $\thresholddistancefactor > 1$.
Furthermore, we only prune loop closure but no odometry edges.

\section{Evaluation}
\newcommand{\paramscautious}{p_\text{cautious}}
\newcommand{\paramsaggressive}{p_\text{aggressive}}
\newcommand{\paramsreference}{p_\text{reference}}

\begin{figure*}[h!]
	\centering
	\begin{subfigure}{0.49\linewidth}
		\includegraphics[width=\linewidth]{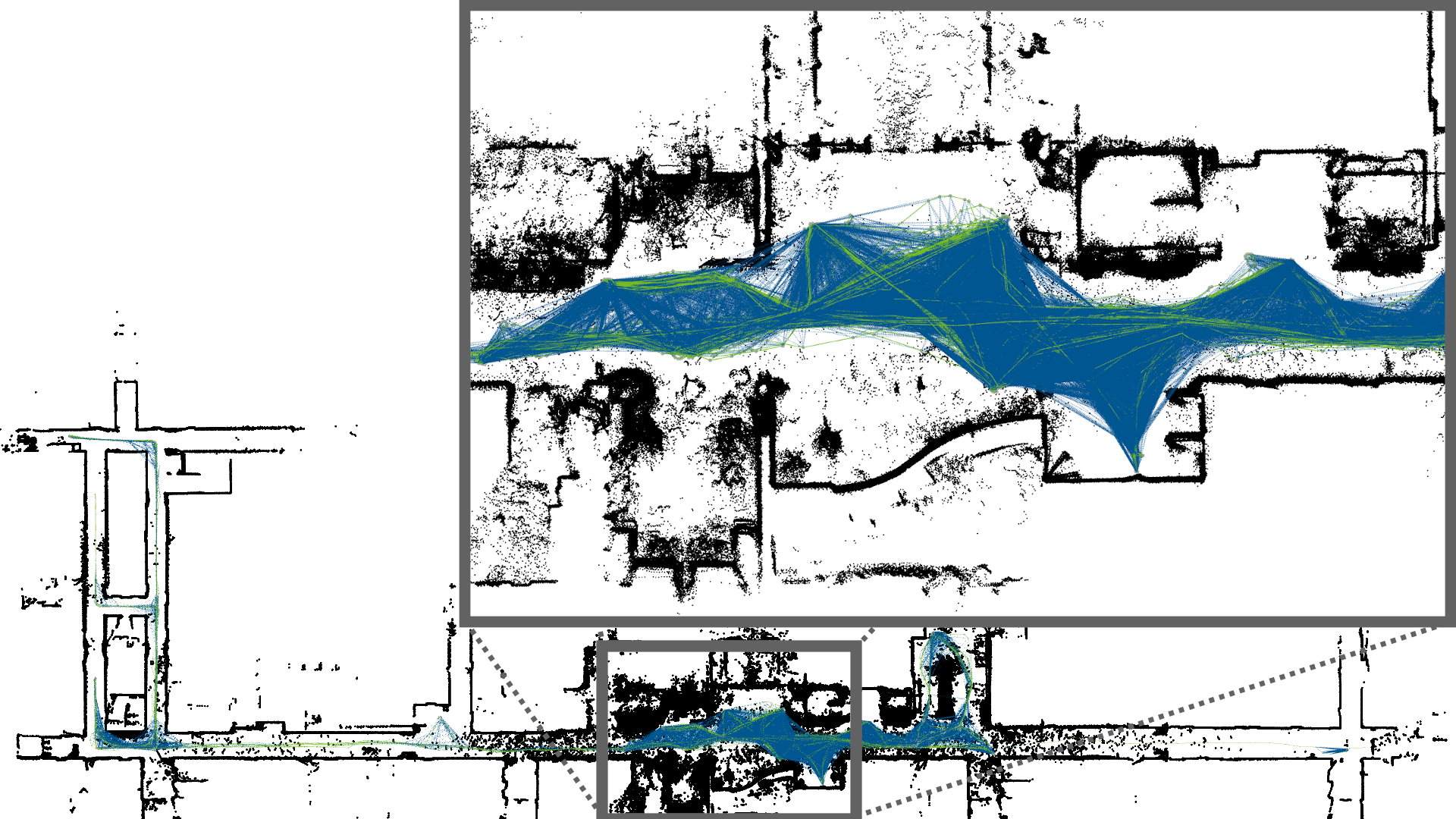}
		\caption{
			$\paramsreference$, \num{8520} vertices, \num{106785} edges.
		}
	\end{subfigure}
	\begin{subfigure}{0.49\linewidth}
		\includegraphics[width=\linewidth]{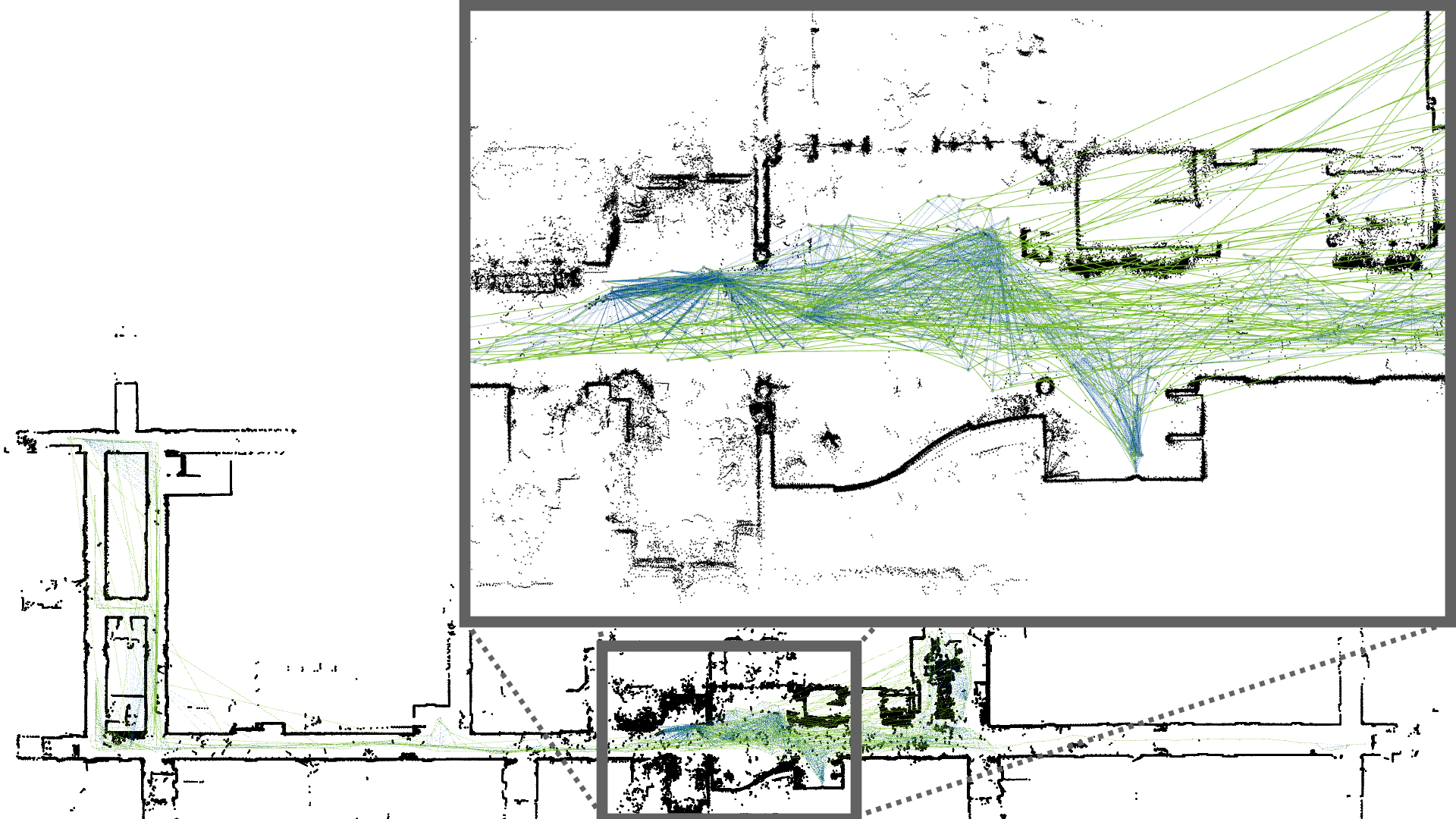}
		\caption{
			$\paramsaggressive$, \num{826} vertices, \num{3606} edges.
		}
	\end{subfigure}
	
	\begin{subfigure}{0.49\linewidth}
		\includegraphics[width=\linewidth, trim=0 350 0 0mm, clip]{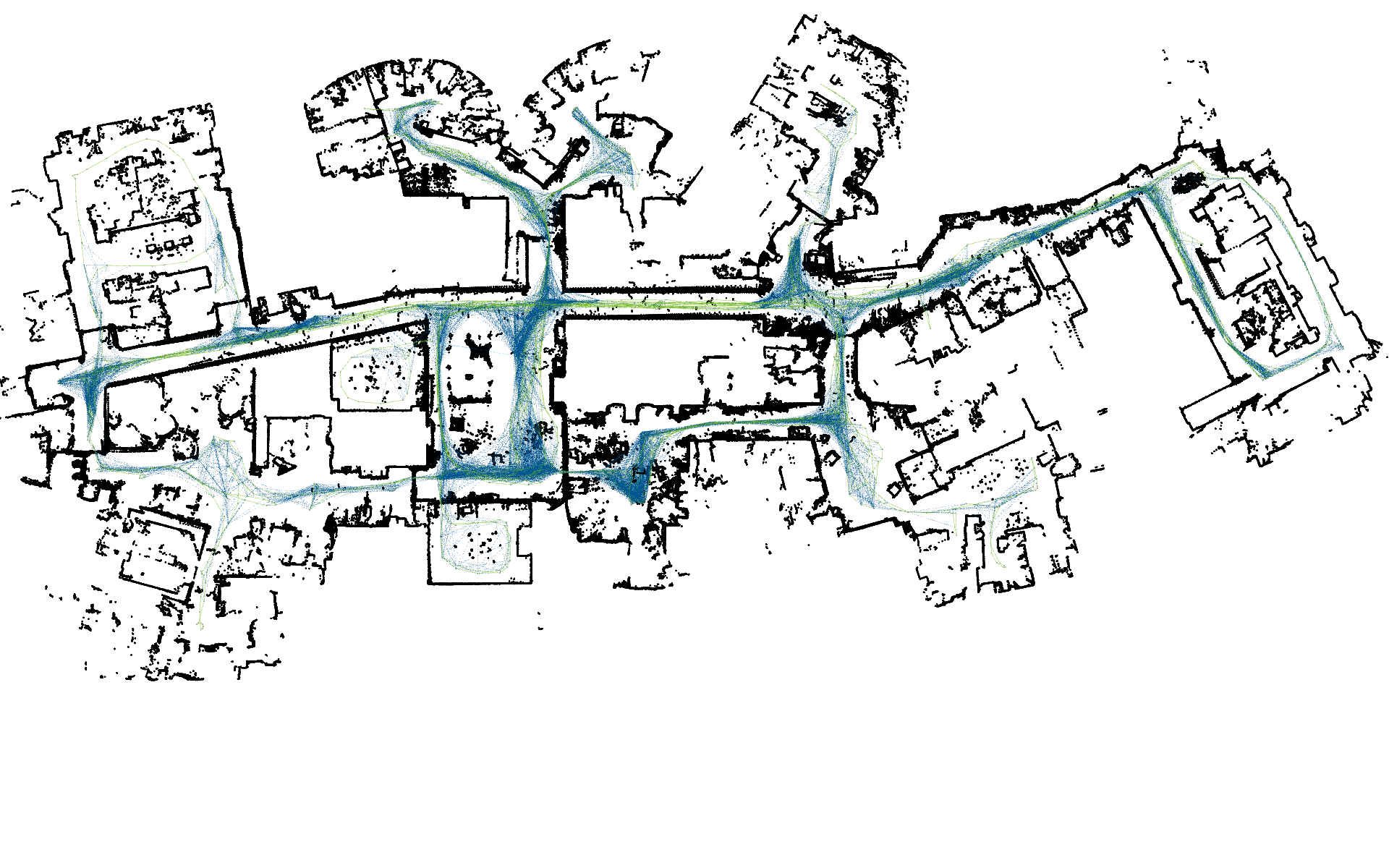}
		\caption{
			$\paramsreference$, \num{3062} vertices, \num{20459} edges.
		}
	\end{subfigure}
	\begin{subfigure}{0.49\linewidth}
		\includegraphics[width=\linewidth, trim=0 350 0 0mm, clip]{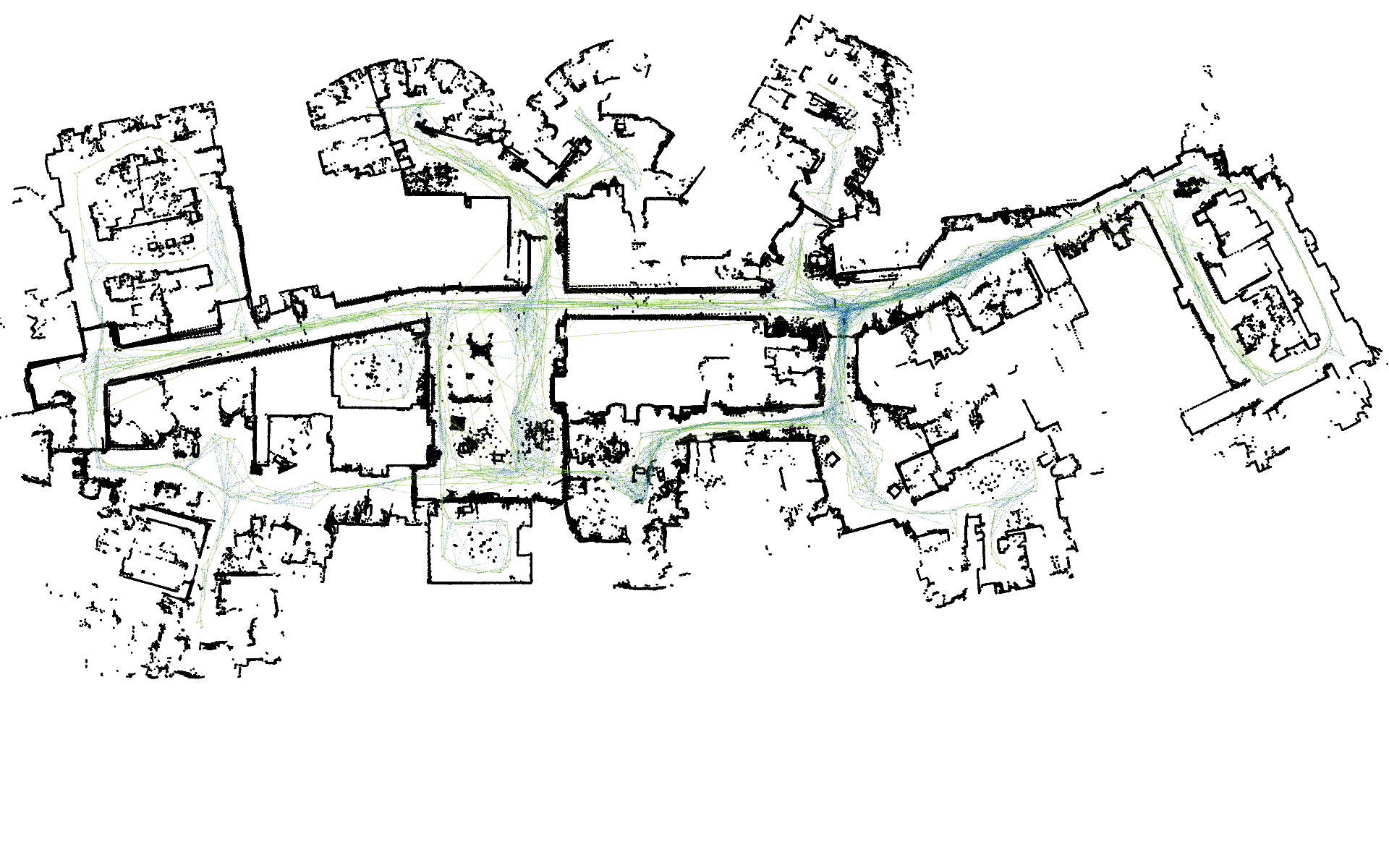}
		\caption{
			$\paramsaggressive$, \num{1199} vertices, \num{3914} edges.
		}
	\end{subfigure}
	\caption{
		Maps of LCAS Strands Care Home (top) and MIT Stata Center (bottom), created with different SLAM settings.
	}
	\label{fig:maps_with_and_without_pruning}
\end{figure*}

\begin{table}
	\centering
	\caption{Pruning parameters used during evaluation.}
	\label{tab:pruning_parameters}
	\begin{tabular}{ lcccccc }
		\toprule
		Config Name  & $\thresholdsid$ & $\sidtruncation$ & $\thresholdvertices$ & $\thresholdodochain$ & $\thresholdedges$ & $\thresholddistancefactor$  \\
		\midrule
		$\paramsaggressive$ & $5.0 $ & 10 & 50 & 50 & 5 & 5.0 \\
		$\paramscautious$ & $15.0 $ & 10 & 50 & 50 & 5 & 5.0 \\
		
		$\paramsreference$ & \multicolumn{6}{c}{No pruning} \\
		\bottomrule
	\end{tabular} 
\end{table}
To evaluate our graph pruning method, we apply it as part of a graph-based SLAM algorithm with 360$^\circ$ LIDAR data and evaluate the SLAM performance depending on different graph pruning parameters (see Table~\ref{tab:pruning_parameters}) on two public datasets (see Table~\ref{tab:datasets}).

Our SLAM algorithm uses g2o~\cite{Kuemmerle2011} for graph optimization.
Vertices in our graph have a single laser scan associated to them.
New vertices are added in regular distances as keyframes; they are connected by odometry edges that encode incremental localization information.
To obtain an odometry edge, we perform early fusion of scan-to-scan matches, wheel odometry, and IMU, if available.
This makes our odometry edges quite reliable w.r.t. outliers and motivates our focus on the odometry chain for this work.
Additionally, we add loop closure edges by matching scans of non-subsequent vertices.
For scan matching, we use the Normal Distributions Transform (NDT)~\cite{Biber2003}.
Place recognition techniques are employed to merge graphs of previous SLAM runs into the current graph.
This allows us to build a consistent map using recordings from multiple days without knowing the robot's starting pose in a shared reference frame beforehand.
Dynamic objects are handled with a technique inspired by~\cite{Underwood2013}.

\subsection{Datasets}
\begin{table}
	\centering
	\caption{Datasets used for evaluation.}
	\label{tab:datasets}
	\begin{tabular}{ llrr }
		\toprule
		\textbf{Dataset} & \textbf{Run} & \textbf{Travelled} & \textbf{Duration} \\
		& & (m) & (hr) \\
		\midrule
		\multirow{4}{*}{\shortstack[l]{\textbf{LCAS Strands} \\ \textbf{Care Home}}}
		& 2016-11-25          & 2634 & 7.68 \\
		& 2016-12-02          & 2092 & 9.32 \\
		& 2016-12-16          & 2081 & 8.00 \\
		\cline{2-4} & Sum     & 6807 & 25.00 \\
		\midrule
		\multirow{6}{*}{\shortstack[l]{\textbf{MIT Stata}\\ \textbf{Center}}}
		& 2012-01-18-09-09-07 & 683  & 0.60 \\
		& 2012-01-25-12-14-25 & 348  & 0.33 \\
		& 2012-01-25-12-33-29 & 239  & 0.23 \\
		& 2012-01-28-11-12-01 & 635  & 0.60 \\
		& 2012-02-02-10-44-08 & 1003 & 0.87 \\
		\cline{2-4} & Sum     & 2908 & 2.63 \\
		\bottomrule
	\end{tabular}
\end{table}

As our first evaluation dataset, we use the Care~Home\footnote{\scriptsize \url{https://lcas.lincoln.ac.uk/nextcloud/shared/datasets/aaf.html}} dataset from the LCAS-STRANDS long-term dataset collection, initially recorded for~\cite{Krajnik2016}. 
We picked the first three days with long trajectories ($>$\SI{2000}{\meter}) that had no recording issues (see Table~\ref{tab:datasets}).
Note that this dataset contains a lot of dynamic objects.
Since it does not come with ground truth information, we run our SLAM algorithm without pruning ($\paramsreference$) and use it as a baseline for comparison of trajectories and vertex positions of the pruned graph.
This way, we can ascertain how much pruning changes the results of the SLAM algorithm and thus degrades the performance.
This experiment design measures how pruning impacts relative SLAM performance rather than the SLAM algorithm's performance as a whole.


As our second evaluation dataset, we use five mapping sessions of the second floor of the MIT~Stata~Center\footnote{\scriptsize\url{http://projects.csail.mit.edu/stata/index.php}}.
We have chosen the same datasets as~\cite{Lazaro2018} (see Table~\ref{tab:datasets}).
We use the laser data from the base of the PR2 robot (a Hokuyo UTM-30LX Laser), the Microstrain 3DM-GX2 IMU and the robot's raw wheel odometry.
This dataset comes with ground truth information,
so in addition to the pruning/no-pruning comparison, we can also evaluate our SLAM algorithm's performance as a whole.

\subsection{Effects on Runtime and Memory}

\begin{figure}
	\centering
	\begin{subfigure}{\linewidth}
		\centering
		\includegraphics[width=0.45\linewidth]{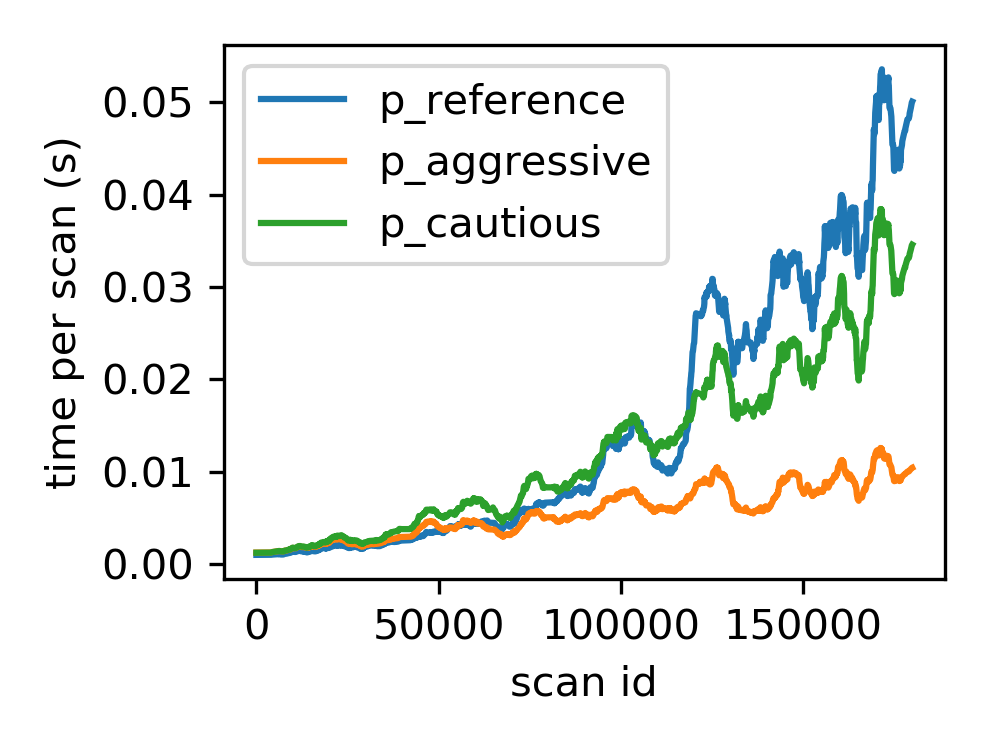}
		\includegraphics[width=0.45\linewidth]{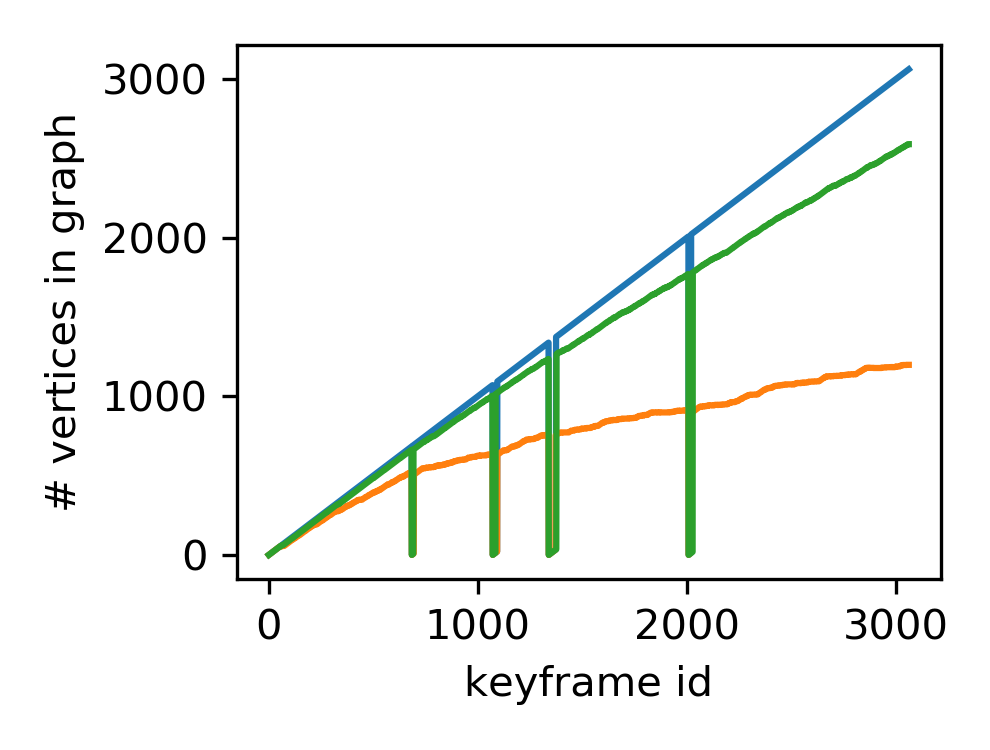}
		\caption{MIT Stata Center}
	\end{subfigure}
	\begin{subfigure}{\linewidth}
		\centering
		\includegraphics[width=0.45\linewidth]{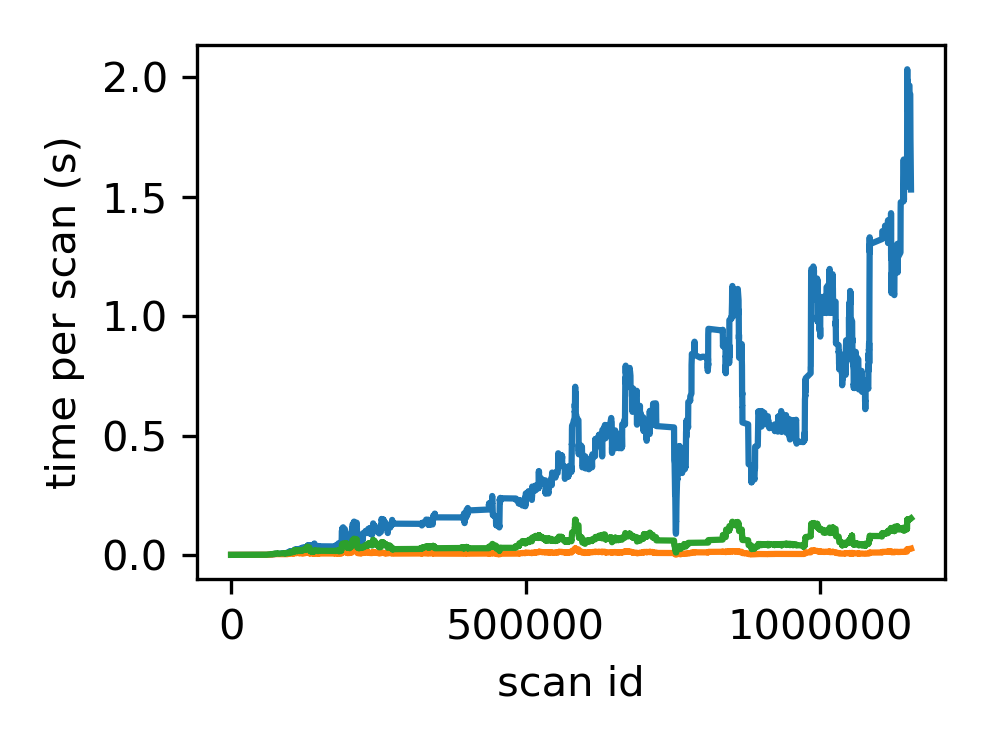}
		\includegraphics[width=0.45\linewidth]{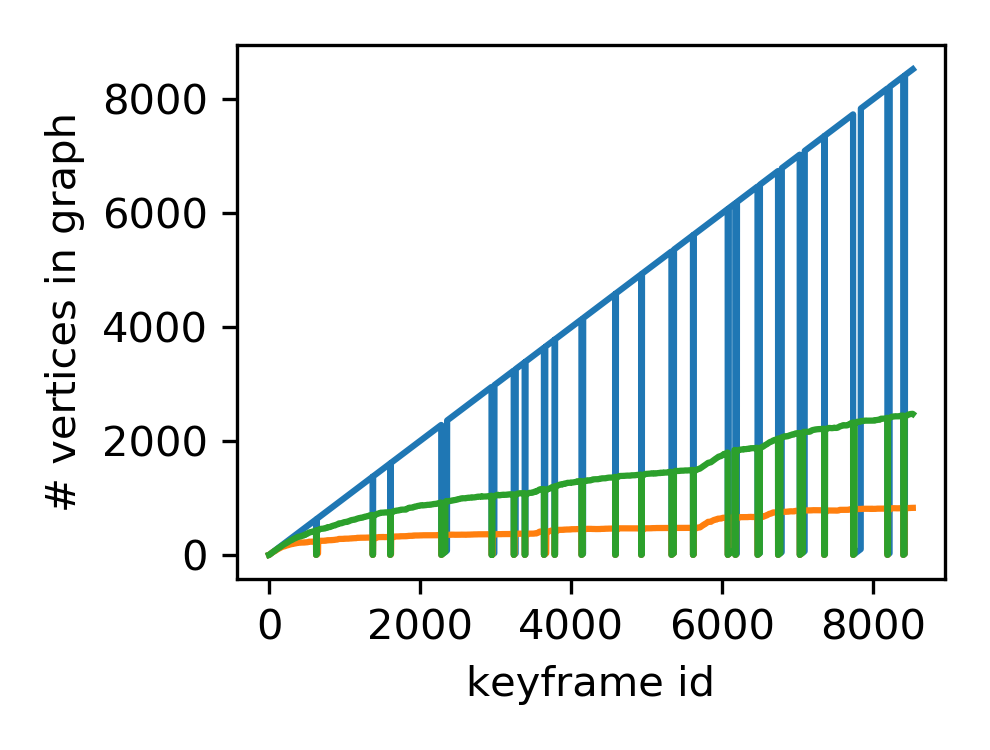}
		\caption{LCAS-STRANDS Care Home}
	\end{subfigure}
	\caption{Time per scan (left) and number of vertices (right) over time for our two evaluation datasets.
		Note that the irregularities in the plot are caused by relocalization after a new mapping run of a dataset has started.}
	\label{fig:time_per_scan_and_vertices_in_graph}
\end{figure}

We run our SLAM algorithm on both datasets (see Figure~\ref{fig:maps_with_and_without_pruning}) and measure the time it takes to process a single scan (see Figure~\ref{fig:time_per_scan_and_vertices_in_graph}).
This includes the duration of scan matching, loop closure search, and the graph optimization.
Note that, especially for bigger graphs, the graph optimization makes up the largest portion of the measured time per scan.
Considering time per scan and the LIDAR's frequency allows us to judge whether the algorithm can run in realtime on a robot.
All presented results were obtained on a single core of an Intel\textregistered~i7-8650U~@~1.90GHz CPU.

In order to assess how pruning improves the memory requirements of our SLAM algorithm, we keep track of the number of vertices in its graph.
Only laserscans that are associated to a vertex are needed long-term.
So, whenever a vertex gets pruned, the associated laserscan can also be deleted.
As visualized on the right side of Figure~\ref{fig:time_per_scan_and_vertices_in_graph}, the number of vertices grows linearly with time when no pruning is performed ($\paramsreference$).
However, if pruning is performed, the number of vertices begins to converge to a level dependent on the $\thresholdsid$ and the size of the environment.
When using the $\paramscautious$ configuration for the MIT Stata Center dataset, this barely starts to happen.
The effect is more evident for the longer LCAS-Strands Care Home dataset, or when using the $\paramsaggressive$ configuration.
Note that we start with a new graph for each individual bag file of a dataset, so the number of vertices drops down to zero at that point in time.
After place recognition succeeds, we merge the new and the old graphs and the number of vertices jumps back up.

\subsection{Effects on Trajectory Error and Map Quality}
\begin{table*}
	\centering
	\caption{SLAM results vs. the ground truth information (top) and vs. reference SLAM configuration (center and bottom).
		Used metrics: Mean trajectory error (TE), map error (ME) and relative map error (RME) and their respective standard deviation (SD).
	}
	\label{tab:slam_errors}
	\begin{tabular}{ llccccccrr }
		\toprule
		\textbf{Dataset} & \textbf{Config}     & \multicolumn{2}{c}{\textbf{TE Mean $\pm$ SD}} & \multicolumn{2}{c}{\textbf{ME Mean $\pm$ SD}} & \multicolumn{2}{c}{\textbf{RME Mean $\pm$ SD}} & \textbf{Runtime} & \textbf{Speedup} \\
		& \textbf{name}       & (m)               & (deg)              & (m)               & (deg)              & (m)                & (deg)                                    &    (s)         &      \\
		\midrule
		\multirow{8}{*}{\shortstack[l]{\textbf{MIT Stata} \\ \textbf{Center}}}
		& \multicolumn{9}{l}{{Comparison with ground truth}} \\
		\cline{2-10}
		& $\paramsreference$  & $0.45  \pm 0.37$  & $1.29  \pm 2.15$  & $0.10  \pm 0.07$  & $0.41  \pm 0.57$  & $0.02  \pm 0.03$  & $0.22  \pm 0.31$  & $\num{1788}$  &  ---    \\
		& $\paramscautious$   & $0.47  \pm 0.42$  & $1.28  \pm 2.19$  & $0.12  \pm 0.09$  & $0.46  \pm 0.84$  & $0.02  \pm 0.03$  & $0.25  \pm 0.36$  & $\num{1178}$  & $1.52$  \\
		& $\paramsaggressive$ & $0.47  \pm 0.33$  & $1.29  \pm 2.17$  & $0.13  \pm 0.11$  & $0.58  \pm 1.02$  & $0.03  \pm 0.04$  & $0.37  \pm 0.50$  & $664$   & $2.69$  \\
		\cmidrule{2-10}
		& \multicolumn{9}{l}{{Comparison with SLAM ($\paramsreference$)}} \\
		\cline{2-10}
		& $\paramsreference$  & ---               & ---               & ---               & ---               & ---               & ---               & $\num{1788}$  &  ---    \\
		& $\paramscautious$   & $0.06  \pm 0.08$  & $0.22  \pm 0.28$  & $0.11  \pm 0.07$  & $0.33  \pm 0.25$  & $0.00  \pm 0.00$  & $0.07  \pm 0.09$  & $\num{1178}$  & $1.52$  \\
		& $\paramsaggressive$ & $0.07  \pm 0.07$  & $0.24  \pm 0.27$  & $0.13  \pm 0.08$  & $0.39  \pm 0.34$  & $0.01  \pm 0.01$  & $0.13  \pm 0.21$  & $\num{664}$   & $2.69$  \\
		\midrule
		\multirow{4}{*}{\shortstack[l]{\textbf{LCAS Strands} \\ \textbf{Care Home}}}
		& \multicolumn{9}{l}{{Comparison with SLAM ($\paramsreference$)}} \\
		\cline{2-10}
		& $\paramsreference$  & ---               & ---               & ---               & ---               & ---               & ---               & $\num{56905}$ & ---     \\
		& $\paramscautious$   & $0.07 \pm 0.13$   & $0.23 \pm 0.31$   & $0.11  \pm 0.10$  & $0.31  \pm 0.25$  & $0.02  \pm 0.02$  & $0.17  \pm 0.07$  & $\num{7109}$  & $8.00$  \\
		& $\paramsaggressive$ & $0.10 \pm 0.21$   & $0.36 \pm 0.39$   & $0.06  \pm 0.08$  & $0.28  \pm 0.25$  & $0.03  \pm 0.04$  & $0.20  \pm 0.22$  & $\num{1399}$  & $40.68$ \\
		\bottomrule
	\end{tabular}
	\vspace{-5mm}
\end{table*}

In the following, we consider several different error measures to assess the accuracy of the SLAM algorithm. 
%
The trajectory error \emph{(TE)} measures the absolute position errors of the robot's trajectory during execution.
We calculate this metric by capturing our SLAM algorithm's pose estimate after each scan.
This pose estimate gets compared to the reference pose estimate or to the ground truth.
This metric does not include corrections to the estimated trajectory that become available later, e.g., through loop closures.
%
The map error \emph{(ME)} shows the absolute position errors of the final map, after all data was processed.
We calculate this metric by comparing the final poses of vertices in the SLAM graph to a reference graph or the ground truth.
%
Both the trajectory error and the map error assess global consistency. 
For example, these metrics will punish small errors in orientation severely, especially if the map is big.
We argue that for most use-cases of lifelong SLAM, local consistency is sufficient.
Therefore, we suggest to focus on the relative map error \emph{(RME)} as proposed by~\cite[Section~III-A]{Burgard2009comparison}.
We calculate this metric by comparing the relative poses between subsequent vertices in the final SLAM graph with the respective relative pose from a reference graph or the ground truth.
Note that this procedure may degenerate to assess global consistency, instead.
This can happen if so many vertices get pruned that subsequent vertices are not close to each other anymore.
In practice, this was not an issue throughout our evaluation, because our pruning thresholds are not that strict.

In Table~\ref{tab:slam_errors}, we compare SLAM results obtained with different configurations (see Table~\ref{tab:pruning_parameters}) to the ground truth of the MIT Stata Center dataset.
Our experiments show that graph pruning results in a speedup of $2.69\times$ at the cost of slightly increased errors on this dataset.
If no pruning is performed, the resource requirements of the SLAM algorithm do not stop growing.
This is naturally more evident for the longer LCAS Strands Care Home dataset, where the speedup is $40.68\times$. Note that the algorithm without pruning is not able to maintain real-time performance towards the end of the run.
When using our pruning method, however, this is not the case.
Also, we compare the results when performing graph pruning with the results of the $\paramsreference$ configuration, which does not perform pruning.
This evaluation is helpful when tuning pruning parameters for specific robots or environments that have different resource requirements, as it does not require ground truth information.
Especially when considering the RME, the errors w.r.t. $\paramsreference$ roughly fit the increase in errors w.r.t. the ground truth for MIT Stata Center.

\section{Conclusion}

We presented a novel graph pruning algorithm for use in lifelong SLAM.
Our key idea is to compute the vertex density using a special density function and to prune vertices in areas where their density is too high.
Furthermore, we proposed an approach to marginalize vertices that is robust to wrong loop closures, which was experimentally validated on synthetic examples where it performs significantly better than the Chow--Liu tree method used by~\cite{Kretzschmar2011}.
Our pruning algorithm was evaluated on two public real-world datasets.
The evaluation shows a significant speedup (2.69$\times$ for the smaller and 40.68$\times$ for the larger dataset) at a moderate cost in terms of accuracy compared to standard graph-based SLAM without pruning (\SI{13}{\centi\meter} and \SI{6}{\centi\meter} map error, respectively).
Compared to the ground truth, the trajectory error increased from \SI{45}{\centi\meter} to \SI{47}{\centi\meter} (or \SI{4.4}{\percent}) when using pruning.

While the proposed algorithm is limited to 2D SLAM, generalization to 3D is straightforward.
Future work may include the consideration of the combined effect of pruning multiple vertices/edges as well as the development of more sophisticated marginalization algorithms that are still robust to incorrect loop closures.

\bibliographystyle{IEEEtran}
\bibliography{../bib/gkbosch}
	
\end{document}